\documentclass{article}

\usepackage{arxiv}
\usepackage{natbib}

\usepackage{hyperref}
\usepackage{cleveref}
\usepackage{booktabs} 
\usepackage[english]{babel}
\usepackage[utf8]{inputenc}
\usepackage{algorithm}
\usepackage[noend]{algpseudocode}
\usepackage{multirow}
\usepackage{soul,color,colortbl}
\usepackage{xcolor}
\usepackage{upgreek}
\usepackage{multirow}
\usepackage{url}
\usepackage{verbatim}\usepackage{subcaption}
\usepackage{xcolor}
\usepackage{booktabs}
\usepackage{upgreek}
\usepackage{mathrsfs}
\usepackage{enumerate}
\usepackage{framed}
\usepackage[framemethod=tikz]{mdframed}
\usepackage{graphicx}
\usepackage{multicol}
\usepackage{graphics}
\usepackage{cuted}
\usepackage{algorithm}
\usepackage[noend]{algpseudocode}
\usepackage{amssymb}
\usepackage{pbox}
\usepackage{amsmath}
\usepackage{pbox}
\usepackage{balance}
\usepackage{flushend}
\usepackage{amsthm}
\usepackage{stmaryrd}
\usepackage{subcaption}

\begingroup
\catcode`[=\active \catcode`]=\active

\gdef[{\llbracket} \gdef]{\rrbracket}

\def\getdelim#1#2#3#4\relax{"#4}
\global\delcode`[=\expandafter\getdelim\llbracket\relax
\global\delcode`]=\expandafter\getdelim\rrbracket\relax
\endgroup

\newcommand{\argmin}{\operatorname{argmin}\, }

\newcommand{\indicator}[1]{\llbracket #1 \rrbracket}
\newcommand{\bias}{r}

\newcommand{\X}{\mathbf{X}}

\newcommand{\Real}{\mathbb{R}}
\newcommand{\mbf}{\mathbf}

\newtheorem{theorem}{Theorem}

\title{Anomaly Detection using One-Class Neural Networks}

\author{
  Raghavendra Chalapathy \\
  University of Sydney,\\
  Capital Markets Co-operative Research Centre (CMCRC)\\
  \texttt{rcha9612@uni.sydney.edu.au} \\
   \And
 Aditya Krishna Menon \\
  Data61/CSIRO and the Australian National University\\
  \texttt{aditya.menon@data61.csiro.au} \\
  \And
 Sanjay Chawla \\
  Qatar Computing Research Institute (QCRI), HBKU\\
  \texttt{schawla@qf.org.qa} \\
  }

\begin{document}
\maketitle

\begin{abstract}
We propose a one-class neural network (OC-NN) model to detect anomalies in complex data sets. OC-NN
combines the ability of deep networks to extract progressively rich representation of data
with the one-class objective of creating a tight envelope around normal  data. The OC-NN approach
breaks new ground for the following crucial reason:  data representation in the hidden layer is
driven by the OC-NN objective and is thus customized for anomaly detection. This is a departure
from other approaches which use a hybrid approach of learning deep features using an autoencoder
and then feeding the features into a separate anomaly detection method like one-class SVM (OC-SVM).
The hybrid OC-SVM approach is sub-optimal because it is unable to influence representational
learning in the hidden layers. A comprehensive set of experiments demonstrate that on complex data
sets (like CIFAR and GTSRB), OC-NN performs on par with state-of-the-art methods and outperformed conventional shallow methods in some scenarios.


\end{abstract}

\maketitle

\section{Anomaly detection: motivation and challenges}
A common need when analysing real-world datasets is determining which instances stand out as being dissimilar to all others. Such instances are known as \emph{anomalies}, and the goal of \emph{anomaly detection} (also known as \emph{outlier detection}) is to determine all such instances in a data-driven fashion~\cite{chandola2007outlier}. Anomalies can be caused by errors in the data but sometimes are indicative of a new, previously unknown, underlying process; in fact Hawkins~\cite{hawkins1980identification} defines an outlier as an observation that {\it deviates so significantly from other observations as to arouse suspicion that it was generated by a different mechanism.} Unsupervised anomaly detection techniques uncover anomalies in an unlabeled test data, which plays a pivotal role in a variety of applications, such as, fraud detection, network intrusion detection and fault diagnosis. One-class Support Vector Machines (OC-SVM)~\cite{scholkopf2002support,tax2004support} are widely used, effective unsupervised techniques to identify anomalies. However, performance of OC-SVM is sub-optimal on complex, high dimensional datasets~\cite{vapnik1998statistical,vishwanathan2003simplesvm,bengio2007scaling}.
From recent literature, unsupervised anomaly detection using deep learning is proven to be very effective~\cite{zhou2017anomaly,chalapathy2017robust}. Deep learning methods for anomaly detection can be broadly classified into model architecture using autoencoders~\cite{andrews2016detecting} and hybrid models~\cite{erfani2016high}. Models involving autoencoders utilize magnitude of residual vector (i,e reconstruction error) for making anomaly assessments. While hybrid models mainly use autoencoder as feature extractor, wherein the hidden layer representations are used as input to traditional anomaly detection algorithms such as one-class SVM (OC-SVM). Following the success of transfer learning~\cite{pan2010survey} to obtain rich representative features, hybrid models have adopted pre-trained transfer learning models to obtain features as inputs to anomaly detection methods. Although using generic pre-trained networks for transfer learning representations is efficient, learning representations from scratch, on a moderately sized dataset, for a specific task of anomaly detection is shown to perform better~\cite{andrews2016transfer}. Since the hybrid models extract deep features using an autoencoder and then feed it to a separate anomaly detection method like OC-SVM, they fail to influence representational learning in the hidden layers. In this paper, we build on the theory to integrate a OC-SVM equivalent objective into the neural network architecture. The OC-NN combines the ability of deep networks to extract progressively rich representation of data alongwith the one-class objective, which obtains the hyperplane to separate all the normal data points from the origin. The OC-NN approach is novel for the following crucial reason:  data representation ,is driven by the OC-NN objective and is thus customized for anomaly detection. We show that OC-NN can achieve comparable or better performance in some scenarios than existing shallow state-of-the art methods for complex datasets, while having reasonable training and testing time compared to the existing methods.

We summarize our main contributions as follows:
\begin{itemize}
\item We derive a new one class neural network (OC-NN) model for anomaly detection. OC-NN uses a one class SVM like loss function to drive the training
of the neural network.
\item We propose an alternating minimization algorithm for learning the parameters of the OC-NN model. We observe that the subproblem of
the OC-NN objective is equivalent to a solving a quantile selection problem.
\item We carry out extensive experiments which convincingly demonstrate  that OC-NN  outperforms other state-of-the-art deep learning approaches
for anomaly detection on complex image and sequence data sets.
\end{itemize}

The rest of the paper is structured as follows. In Section~\ref{sec:background} we provide a detailed survey of related and relevant
work on anomaly detection. The main OC-NN model is developed in Section~\ref{sec:method}. The experiment setup, evaluation metrics and
model configurations are described in Section~\ref{sec:ocnnexperiment-setup}. The results and analysis of the experiments are the focus of
Section~\ref{sec:ocnnexperiment-results}. We conclude in Section~\ref{sec:conclusion} with a summary and directions for future work.

\section{Background and related work on anomaly detection}
\label{sec:background}
\label{sec:related}

Anomaly detection is a well-studied topic in Data Science~\cite{chandola2007outlier,charubook}. Unsupervised anomaly detection aims at discovering
rules to separate normal and anomalous data in the absence of labels. One-Class SVM (OC-SVM) is a popular unsupervised approach to detect anomalies, which constructs
a smooth boundary around the majority of probability mass of data~\cite{Scholkopf:2001}. OC-SVM will be  described in detail in Section~\ref{sec:ocsvm}.
In recent times, several approaches of feature selection and feature extraction methods have been proposed for complex, high-dimensional
data for use with OC-SVM ~\cite{cao2003comparison,neumann2005combined}.
Following the unprecedented success of using deep autoencoder networks, as feature extractors, in tasks as diverse as visual, speech anomaly detection~\cite{chong2017abnormal,marchi2017deep}, several hybrid models that combine feature extraction using deep learning and OC-SVM
have appeared~\cite{sohaib2017hybrid,erfani2016high}. The benefits of leveraging pre-trained transfer learning  representations for anomaly detection in hybrid models was made evident by the results obtained, using two publicly available~\footnote{Pretrained-models:http://www.vlfeat.org/matconvnet/pretrained/.} pre-trained CNN models: ImageNet-MatConvNet-VGG-F (VGG-F) and  ImageNet-MatConvNet-VGG-M (VGG-M)~\cite{andrews2016transfer}. However, these hybrid OC-SVM approaches are decoupled in the sense that the feature learning is task agnostic and not customized for anomaly detecion. Recently a deep model which trains a neural network by minimizing the volume of a hypersphere that encloses the network representations of the data is proposed ~\cite{pmlrv80ruff18a}, our approach differs from this approach by combining the ability of deep networks to extract progressively rich representation of data alongwith the one-class objective, which obtains the hyperplane to separate all the normal data points from the origin.
\subsection{Robust Deep Autoencoders for anomaly detection}
Besides the hybrid approaches which use OC-SVM with deep learning features another approach for anomaly detection is to use deep autoencoders.
Inspired by RPCA~\cite{xu2010robust}, unsupervised anomaly detection techniques such as robust deep autoencoders can be used to separate
normal from anomalous data~\cite{zhou2017anomaly,chalapathy2017robust}.
Robust Deep Autoencoder (RDA) or Robust Deep Convolutional Autoencoder (RCAE) decompose input data $X$ into two parts $X = L_D + S$, where $L_D$ represents the latent representation the hidden layer of the autoencoder. The matrix  $S$ captures noise and outliers which are hard to reconstruct as shown in Equation~\ref{eqn:robust-ae}. The decomposition is carried out by optimizing the objective function shown in Equation~\ref{eqn:robust-ae}.

\begin{equation}
	\label{eqn:robust-ae}
	\min_{\theta, S} + ||L_D - D_{\theta}(E_{\theta}(L_D)) ||_{2}+ \lambda \cdot \| S^T \|_{2,1}
\end{equation}
 \hspace{1.8cm}   $s.t. \hspace{0.2cm} X - L_D -S = 0 $

The above optimization problem is solved using a combination of backpropagation and Alternating Direction Method of Multipliers (ADMM) approach~\cite{boyd2004convex}. In our experiments  we have carried out a detailed comparision between OC-NN and approaches based on robust autoencoders.

\subsection{One-Class SVM for anomaly detection}
\label{sec:ocsvm}

One-Class SVM (OC-SVM) is a widely used approach to discover anomalies in an unsupervised fashion~\cite{scholkopf2002support}. OC-SVMs are a special case of support vector machine, which learns a  hyperplane to separate all the data points from the origin in a reproducing kernel Hilbert space (RKHS) and maximises the distance from this hyperplane to the origin. Intuitively in OC-SVM all the data points are considered as positively labeled instances and the origin
as the only negative labeled instance. More specifically, given a training data $\X$, a set without any class information, and $\Phi(\X)$  a RKHS map function
from the input space to the feature space $F$,  a
hyper-plane or linear decision function $f(\X_{n :})$ in the feature space $F$ is constructed as
$f( \X_{n :} ) = w^T \Phi(\X_{n :}) - \bias$,  to separate as many as possible of the mapped vectors
${ \Phi(\X_{n :}), n: 1,2,...,N}$ from the origin. Here  $w$ is the norm perpendicular to the hyper-plane and $\bias$ is the bias of the hyper-plane. In order to
obtain $w$ and $\bias$, we need to solve the
following optimization problem,

\begin{equation}
\label{eqn:ocsvm-objective}
 \min_{w,\bias} \frac{1}{2} \| w \|_{2}^2+ \frac{1}{\nu} \cdot \frac{1}{N} \sum_{n = 1}^N \max( 0, \bias - \langle w, \Phi(\X_{n :}) \rangle ) - \bias.
\end{equation}

where $\nu \in (0,1)$, is a parameter that
controls a trade off between maximizing the distance of the
hyper-plane from the origin and the number of data points
that are allowed to cross the hyper-plane (the false positives).

\section{From One Class SVM  to One Class Neural Networks}
\label{sec:method}

We now present our one-class Neural Network (OC-NN) model for unsupervised anomaly detection.
The method can be seen as designing a neural architecture using an  OC-SVM equivalent loss function.
Using OC-NN we will be able to exploit and refine features obtained from unsupervised tranfer learning specifically for anomaly
detection. This in turn will it make it possible to discern anomalies in complex data sets where the decision boundary between normal and anomalous
is highly nonlinear.
\vspace{-0.1cm}
\subsection{One-Class Neural Networks (OC-NN)}
\label{sec:oc-nn}
We design a simple feed forward network with one hiden layer having linear or sigmoid activation $g(\cdot )$ and one output node. Generalizations to deeper
architectures is straightforward. The OC-NN objective can be formulated as:
\begin{equation}
	\label{eqn:oc-nn}
	\min_{w, V, \bias} \frac{1}{2} \| w \|_{2}^2 + \frac{1}{2} \| V \|_F^2 + \frac{1}{\nu} \cdot \frac{1}{N} \sum_{n= 1}^N \max( 0, \bias - \langle w, g( V  \X_{n :} ) \rangle ) - \bias
\end{equation}

where $w$ is the scalar output obtained from the hidden to output layer,
$V$ is the weight matrix from input to hidden units.
Thus the \underline{key insight} of the paper is to replace the dot product $\mathbf{\langle w,\Phi(\X_{n :}) \rangle}$ in OC-SVM with the dot product $\mathbf{\langle w,g(V\X_{n :})\rangle}$. This change will make it possible to leverage transfer learning features obtained using an autoencoder and create an additional layer to refine
the features for anomaly detection. However, the price for the change is that the objective becomes non-convex and thus the resulting algorithm
for inferring the parameters of the model will not lead to a global optima.

\vspace{-0.1cm}
\subsection{Training the model}
\label{sec:training}
We can optimize Equation~\ref{eqn:oc-nn} using an alternate minimization approach: We first fix $r$ and optimize for $w$ and $V$. We then use
the new values of $w$ and $V$ to optimize $r$. However, as we  will show, the optimal value of $r$ is just the $\upsilon$-quantile of
the array $\langle w,g(Vx_{n})\rangle$. We first define the objective to solve for $w$ and $V$ as
\begin{equation}
                \label{eqn:minimize_w_V}
                 \underset{w, V}{\argmin} \frac{1}{2} \| w \|_{2}^2 + \frac{1}{2} \| V \|_F^2 + \frac{1}{\upsilon} \cdot \frac{1}{N} \sum_{n = 1}^N \ell( y_n, \hat{y}_n( w, V ) )
                 \end{equation}
                where
                \begin{align*}
                        \ell( y, \hat{y} ) &= \max( 0, y - \hat{y} ) \\
                        y_n       &= \bias \\
                        \hat{y}_n( w, V ) &= \langle w, g( V x_n ) \rangle
                \end{align*}

Similarly the optimization problem for $\bias$ is
\begin{equation}
\label{eqn:minimize_r}
\underset{\bias}{\argmin} \left( \frac{1}{N\upsilon} \cdot \sum_{n = 1}^N \max( 0, \bias - \hat{y}_n ) \right) - r
\end{equation}
\begin{theorem}
Given $w$ and $V$ obtained from solving Equation~\ref{eqn:minimize_w_V}, the solution to Equation~\ref{eqn:minimize_r} is given
by the $\upsilon^{\text{th}}$ quantile of $\{ \hat{y}_n \}_{n = 1}^N$, where
                $$ \hat{y}_n = \langle w, g(V x_n ) \rangle. $$
\end{theorem}
\begin{proof}
 We can rewrite Equation~\ref{eqn:minimize_r} as:
\begin{align*}
\underset{\bias}{\argmin} \left( \frac{1}{N\upsilon} \cdot \sum_{n = 1}^N \max( 0, \bias - \hat{y}_n ) \right) - \left( \bias - \frac{1}{N} \sum_{n = 1}^{N} \hat{y}_n \right) \\
= \underset{\bias}{\argmin} \left( \frac{1}{N\upsilon} \cdot \sum_{n = 1}^N \max( 0, \bias - \hat{y}_n ) \right) - \left( \frac{1}{N} \sum_{n = 1}^{N} \left[ \bias - \hat{y}_n \right] \right) \\
= \underset{\bias}{\argmin} \left( \sum_{n = 1}^N \max( 0, \bias - \hat{y}_n ) \right) - \upsilon \cdot \left( \sum_{n = 1}^{N} \left[ \bias - \hat{y}_n \right] \right) \\
= \underset{\bias}{\argmin} \sum_{n = 1}^N \left[ \max( 0, \bias - \hat{y}_n ) - \upsilon \cdot \left( \bias - \hat{y}_n \right) \right] \\
= \underset{\bias}{\argmin} \sum_{n = 1}^N
\begin{cases} (1 - \upsilon) \cdot \left( \bias - \hat{y}_n \right) & \text{ if } \bias - \hat{y}_n > 0 \\ - \upsilon \cdot \left( \bias - \hat{y}_n \right) & \text{ otherwise}
\end{cases}
\end{align*}
\vspace{-0.1cm}
We can observe that the derivative with respect to $r$ is
$$ F'( r ) = \sum_{n = 1}^N \begin{cases} (1 - \upsilon) & \text{ if } \bias - \hat{y}_n > 0 \\ -\upsilon & \text{ otherwise. } \end{cases} $$
Thus, by F'( r ) = 0 we obtain
\begin{align*}
	(1 - \upsilon) \cdot \sum_{n = 1}^N \indicator{ \bias - \hat{y}_n > 0 } &= \upsilon \cdot \sum_{n = 1}^N \indicator{ \bias - \hat{y}_n \leq 0 } \\
	&= \upsilon \cdot \sum_{n = 1}^N (1 - \indicator{ \bias - \hat{y}_n > 0 }) \\
	&= \upsilon \cdot N - \upsilon \cdot \sum_{n = 1}^N \indicator{ \bias - \hat{y}_n > 0 },
\end{align*}
or
\begin{equation}
 \frac{1}{N} \sum_{n = 1}^N \indicator{ \bias -\hat{y}_n > 0 }= \frac{1}{N} \sum_{n = 1}^N \indicator{ \hat{y}_n < r } = \upsilon
\end{equation}\\
This means we would require the $\nu^{\text{th}}$ quantile of $\{ \hat{y}_n \}_{n = 1}^N$.
\end{proof}
\vspace{-0.3cm}
\subsection{OC-NN Algorithm}
\label{sec:algorithm}
We summarize the solution in Algorithm~\ref{alg1}. We initialize $\bias^{(0)}$ in Line 2. We learn the parameters($w,V$) of the neural network
using the standard Backpropogation(BP) algorithm (Line 7). In the experiment section, we will train the model using features extracted from
an autoencoder instead of raw data points. However this has no impact on the OC-NN algorithm. As show  in Theorem 3.1, we solve for $\bias$
using the $\upsilon$-quantile of the scores $\langle y_{n} \rangle$. Once the convergence criterion is satisfied, the data points are
labeled normal or anomalous using the decision function $S_{n} = sgn(\hat{y}_{n} - r)$.

\begin{algorithm}
\caption{one-class neural network (OC-NN) algorithm}\label{alg:oc-nn}
\label{alg1}
\begin{algorithmic}[1]
\State{
\textbf{Input:} Set of points $\X_{n :},\hspace{0.2cm} n: 1,...,N$
}
\State{
\textbf{Output:} A Set of decision scores $S_{n :}=\hat{y}_{n :}$, n: 1,...,N for X
}

\State Initialise $\bias^{(0)}$
\State $t \gets 0$
\While{(no convergence achieved)}
\State Find $(w^{(t+1)}, V^{(t+1)})$ \Comment Optimize Equation~\ref{eqn:minimize_w_V} using BP.
\State $r^{t+1} \gets \nu^{\text{th}}$ quantile of $\{ \hat{y}_n^{t+1} \}_{n = 1}^N$
\State $t \gets t + 1$
\EndWhile \label{endwhile}
\State \textbf{end}
\State Compute decision score $S_{n :}=  \hat{y}_n -r  \mbox{ for each }  \X_{n :}$
\If{($S_{n :}$ $\geq$ 0)}
    \State $\X_{n :}$ is normal point
   \Else
    \State $\X_{n :}$ is anomalous
\EndIf

\State \textbf{return} $\{S_n\}$

\end{algorithmic}
\end{algorithm}
%
\noindent
{\bf Example:} We give a small example to illustrate that the minimum of the function
\[f(r) =  \left( \frac{1}{N\upsilon} \cdot \sum_{n = 1}^N \max( 0, \bias - y_n ) \right) - r\] occurs at the
the $\upsilon$-quantile of the set $\{y_{n}\}$. \\

Let $y =\{1,2,3,4,5,6,7,8,9\}$ and  $\upsilon = 0.33$. Then the minimum will occur at $f(3)$ as detailed in
the table below.
\[
\begin{array}{|r|l|r|}  \hline \hline
r & f(r) (\mbox{expr}) & f(r) (\mbox{value}) \\ \hline \hline
1 & \frac{1}{9*.33}[0 + 0 +\ldots 0] - 1 & -1.00 \\ \hline
2 & \frac{1}{9*.33}[1 + 0 +\ldots 0] - 2 & -1.67 \\ \hline
\rowcolor{green!20}
3 & \frac{1}{9*.33}[2 + 1 +\ldots 0] - 3 & -1.99  \\ \hline
4 & \frac{1}{9*.33}[3 + 2 + 1 +\ldots 0] - 4 & -1.98  \\ \hline
5 & \frac{1}{9*.33}[4 + 3 + 2 + 1\ldots 0] - 5 & -1.63  \\ \hline
6 & \frac{1}{9*.33}[5 + 4  + 3 + 2 + 1\ldots 0] - 6 & -0.94  \\ \hline
7 & \frac{1}{9*.33}[6 + 5  + 4 + 3 + 2 + 1\ldots 0] - 7 & 0.07  \\ \hline
8 & \frac{1}{9*.33}[7 + 6  + 5 + 4 + 3 + 2 + 1\ldots 0] - 8 & 1.43  \\ \hline
9 & \frac{1}{9*.33}[8 + 7  + 6 + 4 + \ldots 0] - 9 & 3.12  \\ \hline
\end{array}
\]

\section{Experimental Setup}
\label{sec:ocnnexperiment-setup}

In this section, we show the empirical effectiveness of OC-NN formulation over the state-of-the-art methods on real-world data. Although our method is applicable in any context where autoencoders may be used for feature representation, e.g., speech.  Our primary focus will be on non-trivial high dimensional images.
\subsection{Methods compared}
\label{sec:methods_compared}
We compare our proposed one-class neural networks (OC-NN) with the following state-of-the-art methods for anomaly detection:
\let\labelitemi\labelitemii
\begin{itemize}{}
    \item \textbf{OC-SVM -SVDD} as per formulation in ~\cite{scholkopf2002support}
    \item \textbf{Isolation Forest} as per formulation in ~\cite{liu2008isolation}.
    \item \textbf{Kernel Density Estimation (KDE)} as per formulation in ~\cite{parzen1962estimation}.
    \item \textbf{Deep Convolutional Autoencoder (DCAE)} as per formulation in ~\cite{masci2011stacked}.
    \item \textbf{AnoGAN} as per formulation in ~\cite{radford2015unsupervised}.
    \item \textbf{Soft-Bound and One Class Deep SVDD} as per formulation in ~\cite{pmlrv80ruff18a}.
    \item \textbf{Robust Convolutional Autoencoder (RCAE)} as per formulation in ~\cite{chalapathy2017robust}.
    \item \textbf{One-class neural networks (OC-NN) \footnote{\url{https://github.com/raghavchalapathy/oc-nn}}}, our proposed model as per Equation \ref{eqn:oc-nn}.
\end{itemize}
We used Keras~\cite{chollet2015keras} and  TensorFlow~\cite{abadi2016tensorflow} for the implementation of OC-NN, DCAE, RCAE \footnote{\url{https://github.com/raghavchalapathy/rcae}}
For OC-SVM\footnote{\url{http://scikit-learn.org/stable/auto_examples/svm/plot_oneclass.html}} and Isolation Forest\footnote{\url{http://scikit-learn.org/stable/modules/generated/sklearn.ensemble.IsolationForest.html}}, we used publicly available implementations.

\begin{table}[!t]
    \centering
    \renewcommand{\arraystretch}{1.25}
    \setlength{\tabcolsep}{6pt}
    \begin{tabular}{@{}llll@{}}
        \toprule
        \toprule
        Dataset & \# instances & \# anomalies & \# features \\
        \toprule
        {\tt Synthetic}       & 190   & 10                                        & 512 \\
        {\tt MNIST}           & single class   & 1\%  ( from all class)            & 784 \\
        {\tt CIFAR$-$10}      & single class    & 10\% ( from all class)           &3072 \\
        {\tt GTSRB }           & 1050 (stop signs )     & 100 (boundary attack)  &3072 \\
        \bottomrule
    \end{tabular}
    \caption{Summary of datasets used in experiments.}
    \label{tbl:datasets}
    \vspace{-\baselineskip}
\end{table}

\subsection{Datasets}
We compare all methods on synthetic and four real-world datasets as summarized below :
\begin{itemize}
    \item {\tt Synthetic Data}, consisting of 190 normal data points 10 anomalous points drawn from normal distribution with dimension $512$.
    \item {\tt MNIST}, consisting of 60000 $28\times28$ grayscale images of handwritten digits  in 10 classes, with 10000  test images~\cite{lecun2010mnist}.
    \item {\tt GTSRB }, are $32\times32$ colour images comprising of  adversarial boundary attack on stop signs boards~\cite{stallkamp2011german}.
    \item {\tt CIFAR$-$10} consisting of 60000 $32\times32$ colour images in 10 classes, with 6000 images per class~\cite{krizhevsky2009learning}.
\end{itemize}

For each dataset, we perform further processing to create a well-posed anomaly detection task, as described in the next section.
\vspace{-0.3 cm}

\subsection{Evaluation of Models}
\label{sec:evaluationOfModels}
\subsubsection{{Baseline Model Parameters:}}
\label{sec:baselinemodel_parameters.}
The proposed OC-NN method is compared with several state-of-the-art baseline models as illustrated in Table~\ref{tab:ocnn_results}. The model parameters of shallow baseline methods are used as per implementation in~\cite{pmlrv80ruff18a}. Shallow Baselines (i) Kernel OC-SVM/SVDD with
Gaussian kernel. We select the inverse length scale $\gamma$ from
$\gamma$  $\in$ {$2^{ - 10}$ , $2^{ - 9}$, . . . , $2^{ - 1}$ via grid search using the performance on a small holdout set (10) \% of randomly drawn test samples). We run all experiments for $\nu = 0.1$ and report the better result. (ii) Kernel density estimation (KDE). We select the bandwidth \textit{h} of the Gaussian kernel from \textit{h} $ \in {2^{0.5} , 2 ^1, . . . , 2^5}$ via 5-fold cross-validation using  the log-likelihood score. (iii) For the  Isolation Forest (IF) we set the number of trees to \textit{t} = 100 and the sub-sampling size to $\Psi = 256$, as recommended in the original work~\cite{pmlrv80ruff18a}.
\vspace{-0.4cm}
\subsubsection{ {Deep Baseline Models:}}
We compare OC$-$NN models to  four deep approaches described Section ~\ref{sec:methods_compared}.
We choose to train DCAE using the Mean sqaure error (MSE) loss since our experiments are on image data. For the DCAE encoder, we employ the same network architectures as we use for Deep SVDD, RCAE, and OC-NN models. The decoder is then constructed symmetrically, where we substitute max-pooling with upsampling. For AnoGAN we follow the implementation as per ~\cite{radford2015unsupervised} and set
the latent space dimensionality to $256$.
For we Deep SVDD, follow the implementation as per ~\cite{pmlrv80ruff18a} and employ a  two phase learning rate schedule (searching + fine tuning) with initial learning rate $\eta$ = $10^{ - 4}$  and subsequently $\eta$ = $10^{ - 5}$. For DCAE we train 250 + 100 epochs, for Deep SVDD 150 + 100. Leaky ReLU activations are used with leakiness $\alpha=0.1$. For RCAE we train the autoencoder using the robust loss and follow the parameter settings as per formulation in ~\cite{chalapathy2017robust}.

\subsubsection{ {One-class neural Networks (OC-NN):}}
\label{model_architecture}
In OC-NN technique firstly, a deep autoencoder is trained to obtain the representative features of the input as illustrated in Figure~\ref{fig:model-architecture}(a). Then, the encoder layers of this pre-trained autoencoder is copied and fed as input to the feed-forward network with one hidden layer as shown in Figure~\ref{fig:model-architecture}(b). The summary of feed forward network architecture's used for various datasets is presented in Table~\ref{tbl:feed-forward-OC-NN}. The weights of encoder network are not frozen (but trained) while we learn feed-forward network weights, following the algorithm summarized in Section~\ref{sec:algorithm}. A feed-forward neural network consisting of single hidden layer, with  linear activation functions produced the best results, as per Equation~\ref{eqn:oc-nn}.  The optimal value of parameter $\nu$ $\in$ ${[0, 1]}$ which is equivalent to the percentage of anomalies for each data set, is set according to respective outlier proportions.

\vspace{-0.2cm}

\begin{figure}[!t]
     \begin{subfigure}[b]{1\textwidth}
   \centering
   {\includegraphics[scale=0.50]
{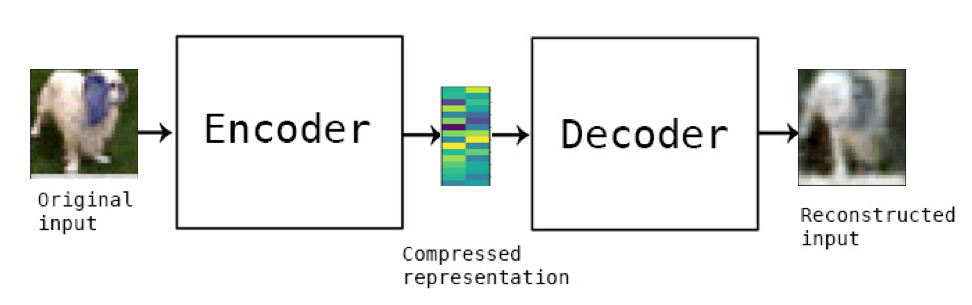}}
\caption{Autoencoder.}
        \end{subfigure}%
        \hfill
        \vspace{2mm}
     \begin{subfigure}[b]{1\textwidth}
\centering
   {\includegraphics[scale=0.50]{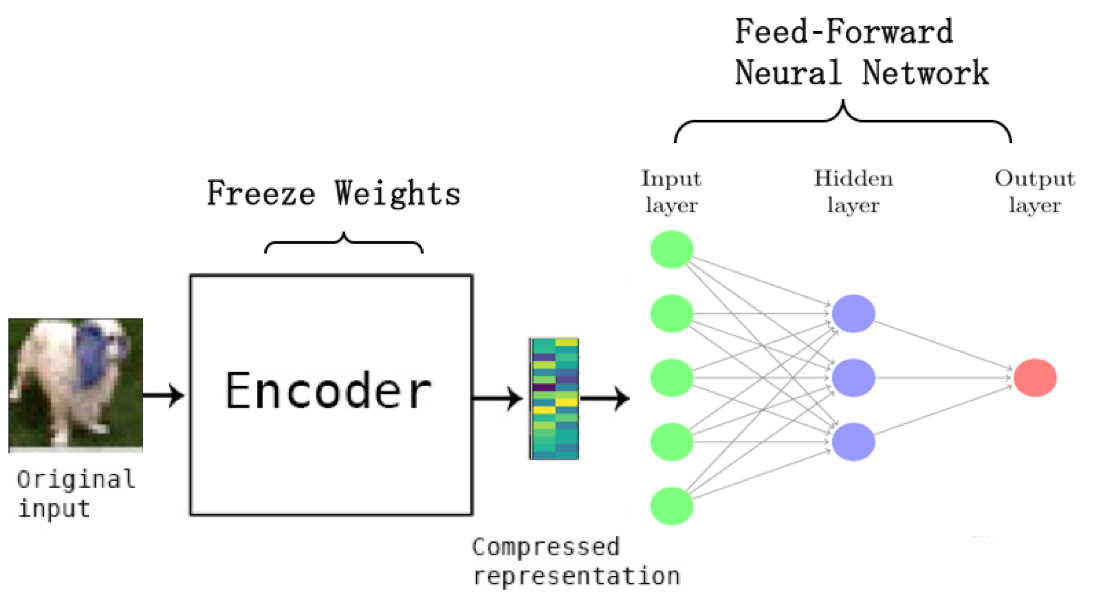}}
 \caption{One-class neural networks.}
        \end{subfigure}%
    \caption{
     Model architecture of Autoencoder  and the proposed one-class neural networks (OC-NN).
    }
    \label{fig:model-architecture}
\end{figure}

\begin{table}[!t]
    \centering
    \renewcommand{\arraystretch}{1.25}
    \setlength{\tabcolsep}{6pt}
    \begin{tabular}{@{}llll@{}}
        \toprule
        \toprule
         Dataset& \#    Input (features) & \# hidden layer( or output) \# optional layer  \\
        \toprule
        {\tt Synthetic }      & 512     & 128   & 1 \\
        {\tt MNIST}           & 32      & 32    & None \\
        {\tt CIFAR$-$10}      & 128     & 32    & None \\
        {\tt GTSRB }          & 128     & 32    & 16 \\
        \bottomrule
    \end{tabular}
    \caption{Summary of best performing feed-forward network architecture's used in OC-NN model for experiments.}
    \label{tbl:feed-forward-OC-NN}
    \vspace{-\baselineskip}
\end{table}



\section{Experimental Results}
\label{sec:ocnnexperiment-results}

In this section, we present empirical results produced by OC-NN model on synthetic and real data sets. We perform a comprehensive set of experiments to demonstrate that on complex data sets OC-NN performs on par with state-of-the-art methods and outperformed conventional shallow methods in some scenarios.
\subsection{ Synthetic Data}
Single cluster consisting of 190 data points using $\mu=0$ and $\sigma = 2$ are generated as normal points , along with 10 anomalous points drawn from normal distribution with $\mu=0$ and $\sigma = 10$ having  dimension $d=512$. Intuitively, the latter normally distributed points are treated as being anomalous, as the corresponding points have different distribution characteristics to the majority of the training data. Each data point is flattened as a row vector, yielding a 190 $\times$ 512 training matrix and 10 $\times$ 512 testing matrix. We use a simple feed-forward neural network model with one-class neural network  objective as per Equation~\ref{eqn:oc-nn}.\\
\textbf{Results}.
From Figure~\ref{fig:synthetic-histogram}, we see that it is a near certainty for all $10$ anomalous points, are accurately identified as outliers with decision scores being negative for anomalies. It is evident that, the OC-NN formulation performs on par with classical OC-SVM in detecting the anomalous data points. In the next sections we apply to image and sequential data and demonstrate its effectiveness.

\begin{figure}
    \centering
    \includegraphics[scale=0.38]{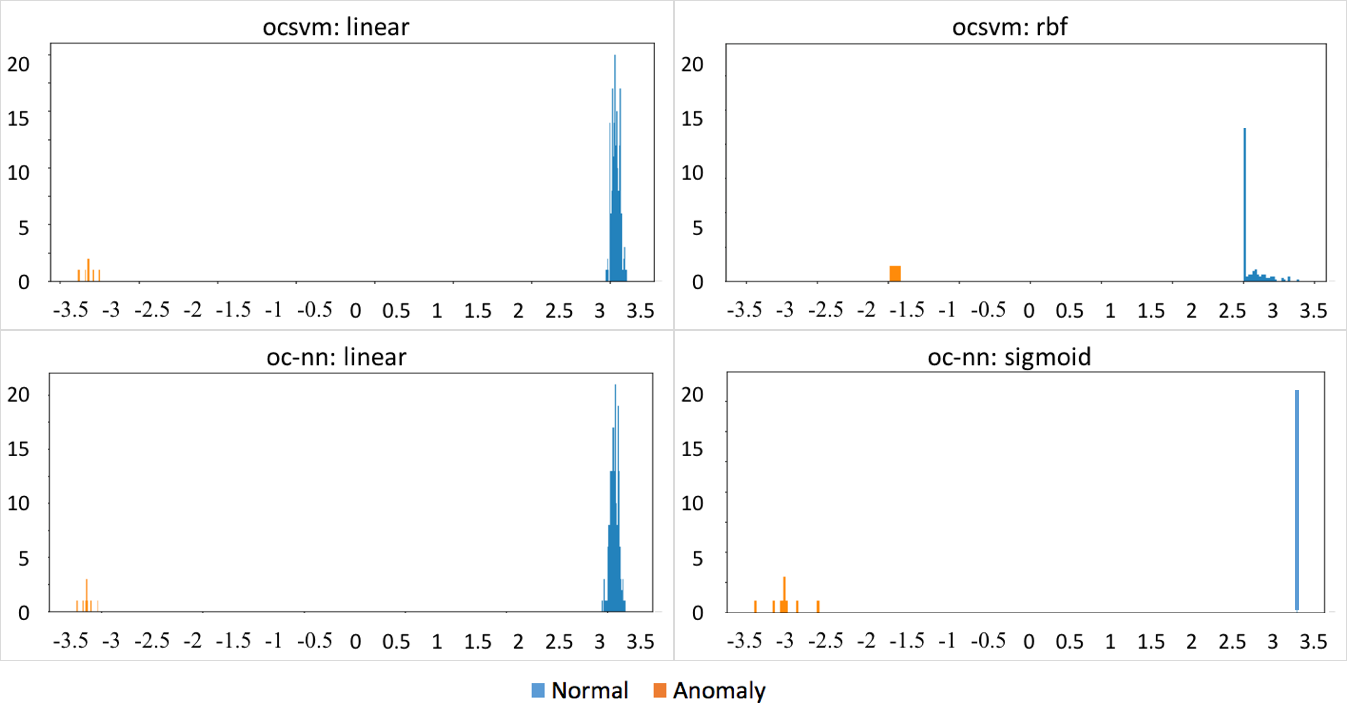}
    \caption{Decision Score Histogram of anomalous vs normal data points, {\tt synthetic} dataset.}
    \label{fig:synthetic-histogram}
\end{figure}

  \begin{figure}[htp]
    \centering
    \subcaptionbox{Normal samples.\label{fig3:a}}{\includegraphics[height=8cm,width=1.6in]{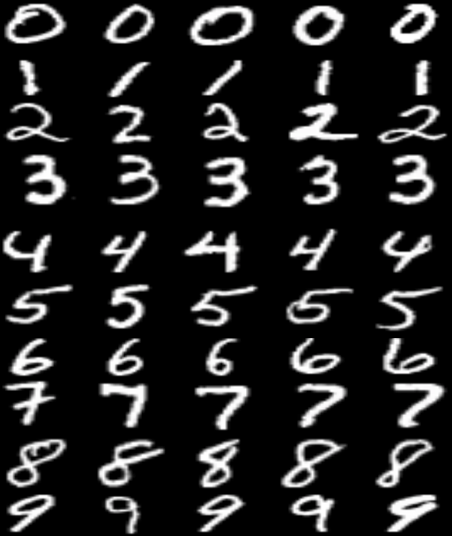}}\hspace{1em}%
    \subcaptionbox{In-class anomalies.\label{fig3:b}}{\includegraphics[height=8cm,width=1.6in]{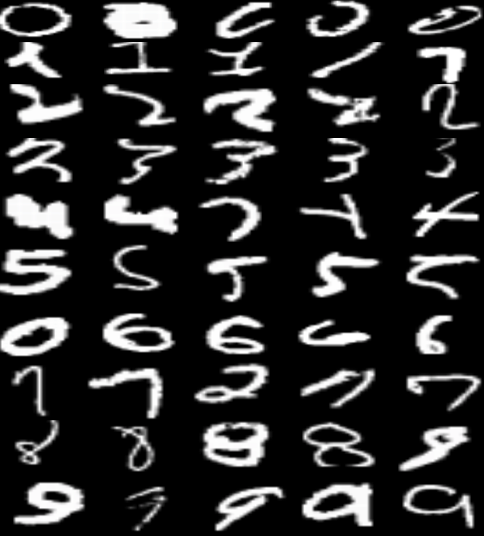}}
    \caption{MNIST Most normal and in-class anomalous MNIST digits detected by RCAE.}
    \label{fig:mnistRCAEresults}
  \end{figure}

    \begin{figure}[htp]
      \centering
      \subcaptionbox{Normal samples.\label{fig3:a}}{\includegraphics[height=8cm,width=1.6in]{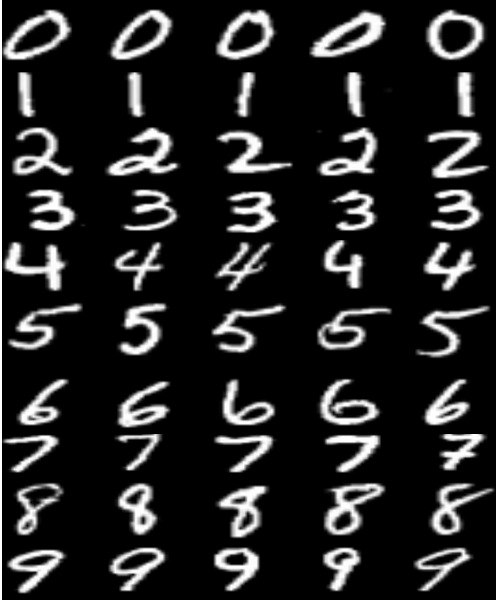}}\hspace{1em}%
      \subcaptionbox{In-class anomalies.\label{fig3:b}}{\includegraphics[height=8cm,width=1.6in]{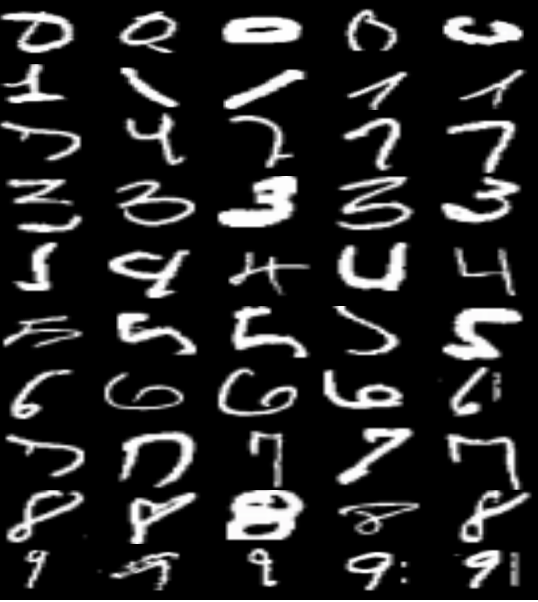}}
      \caption{MNIST Most normal and in-class anomalous MNIST digits detected by OC-NN.}
            \label{fig:mnistOCNNresults}
    \end{figure}

\subsection{ OC-NN on MNIST and CIFAR-10}
\textbf{Setup}
Both MNIST and CIFAR-10  have ten different classes from which we construct one-class classification dataset.  One of the classes is the normal class and
samples from the remaining classes represent anomalies. We use the original training and test splits in our
experiments and only train with training set examples from
the respective normal class. This produces normal instances set of sizes
of n $\approx$ 6000 for MNIST and n $\approx$ 5000 for CIFAR-10 respectively.
Both MNIST and CIFAR-10 test sets have $10000$ samples per class.
The training set comprises of normal instances for each class along with anomalies sampled from
all the other nine classes. The number of anomalies used within training set consists of 1\% normal class for MNIST and  10\% of normal class for CIFAR-10 dataset respectively. We pre-process all images with global contrast normalization using the $L1$
norm and finally rescale to [0, 1] via min-max-scaling.\\
\textbf{Network architectures} For both MNIST and CIFAR-10 datasets, we employ LeNet type
CNNs, wherein each convolutional module consists of a convolutional layer followed by leaky ReLU activations and $2 \times 2$ max-pooling. On MNIST, we use a CNN with two modules, $8\times(5\times5\times1)$-filters followed by $4\times(5\times5\times1)$-filters, and a final dense layer of 32 units. On CIFAR-10,
we use a CNN with three modules, $32 \times (5 \times 5 \times3)$-filters,
$64\times(5\times5\times3)$-filters, and $128\times(5\times5\times3)$-filters, followed
by a final dense layer of $128$ units. We use a batch size of
$200$ and set the weight decay hyperparameter to $ \lambda= 10^{ - 5}$.\\
{\textbf{Results:}}
Results are presented in Table ~\ref{tab:ocnn_results}. RCAE
clearly outperforms both its shallow and deep competitors
on MNIST. On CIFAR-10 the results are convincing but for certain classes the shallow baseline methods outperform the deep models. OC-NN, Soft and One-class Deep SVDD however, shows an overall robust performance. On CIFAR-10 classes such as $AUTOMOBILE$, $BIRD$ and $DEER$ which have less global contrast, as illustrated in Table~\ref{tab:ocnn_results} (indicated in blue color)  OC-NN seems to outperform the shallow methods, Soft and One-class Deep SVDD methods, this is indicative of their future potential on similar data instances.  It is interesting to note that shallow OCSVM/SVDD and KDE perform better than deep methods on two of the ten CIFAR-10 classes. We can see that normal examples of the classes  such as $FROG$ and $TRUCK$ on which OCSVM/SVDD performs best as illustrated in Figure~\ref{fig:ocsvmresults} seem to have strong global structures. For example, TRUCK images are mostly divided horizontally into street and sky, and  $FROG$ have similar colors globally. For these classes,  the performance significantly depends on choice of network architecture. Notably, the One-Class Deep SVDD performs slightly better than its soft-boundary counterpart on both datasets. Figure ~\ref{fig:mnistRCAEresults},  Figure ~\ref{fig:cifar10_rcae_results}  and Figure ~\ref{fig:mnistOCNNresults}, Figure ~\ref{fig:cifar10_ocnn_results} show examples of the most normal and most anomalous in-class samples detected by RCAE and OC-NN for MNIST and CIFAR-10 dataset respectively.

\begin{figure}[htp]
    \centering
    \subcaptionbox{Normal samples.\label{fig3:a}}{\includegraphics[width=1.6in]{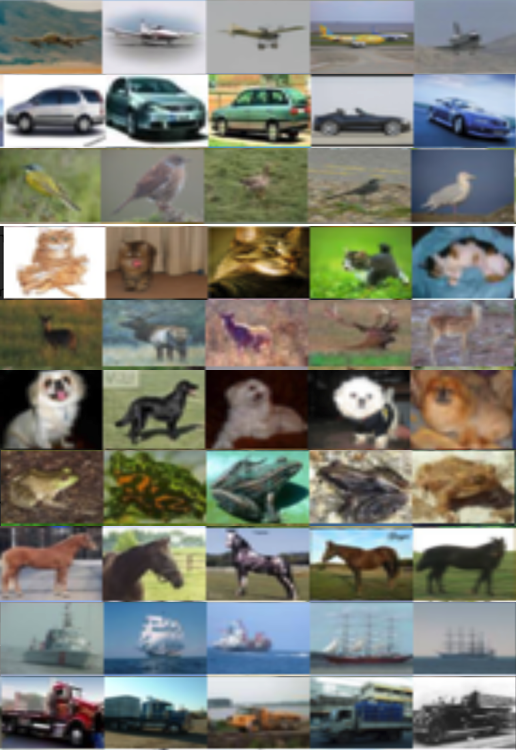}}\hspace{1em}%
    \subcaptionbox{In-class anomalies.\label{fig3:b}}{\includegraphics[width=1.6in]{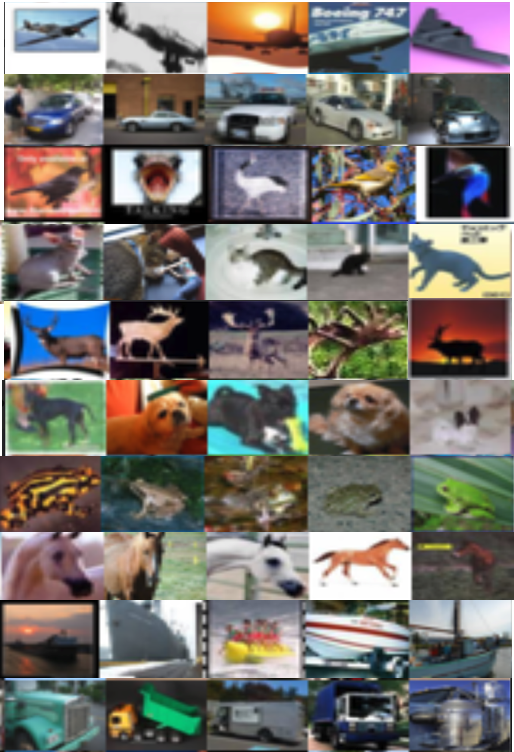}}
  \caption{CIFAR-10 Most normal and in-class anomalous CIFAR10 digits detected by RCAE. }
  \label{fig:cifar10_rcae_results}
\end{figure}

\begin{figure}[htp]
      \centering
      \subcaptionbox{Normal samples.\label{fig3:a}}{\includegraphics[width=1.6in]{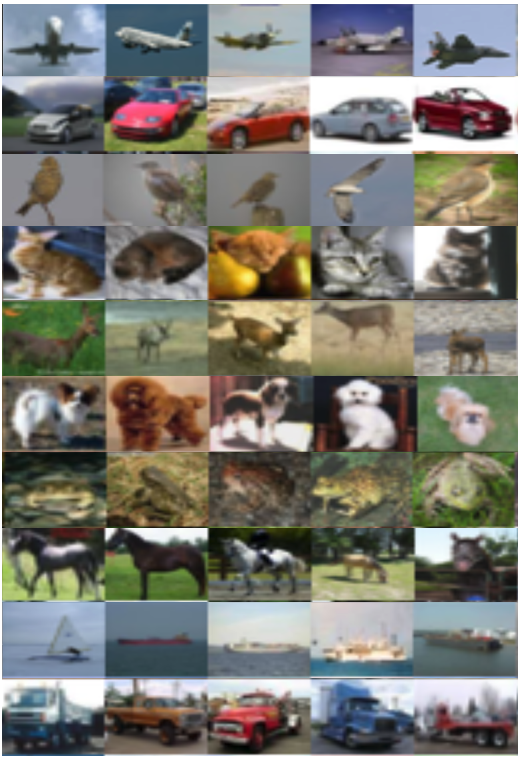}}\hspace{1em}%
      \subcaptionbox{In-class anomalies.\label{fig3:b}}{\includegraphics[width=1.6in]{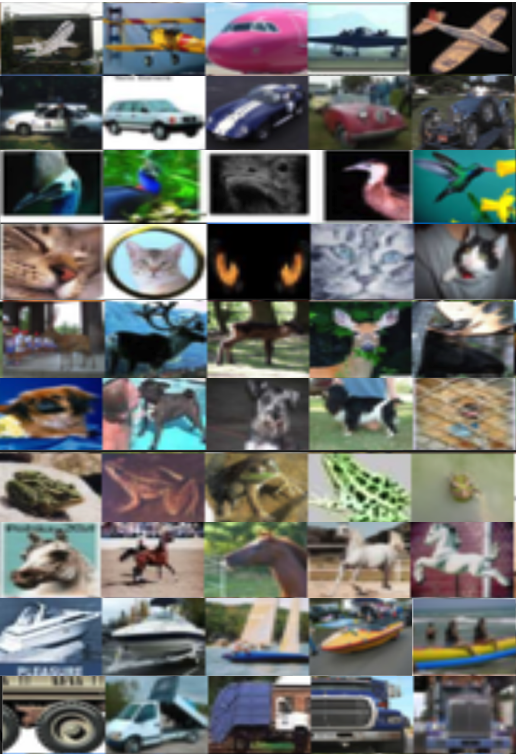}}
    \caption{CIFAR-10 Most normal and in-class anomalous CIFAR-10 digits detected by OC-NN. }
    \label{fig:cifar10_ocnn_results}
\end{figure}

\begin{figure}[htp]
          \centering
          \subcaptionbox{Normal samples.\label{fig3:a}}{\includegraphics[width=1.6in]{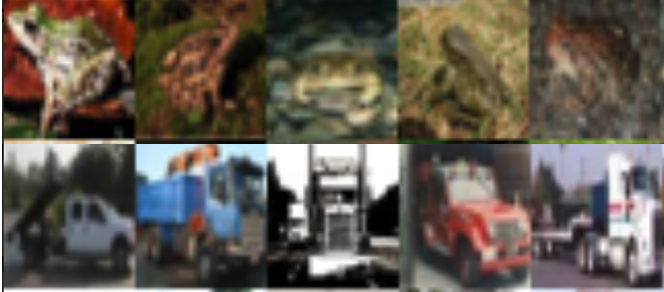}}\hspace{1em}%
          \subcaptionbox{In-class anomalies.\label{fig3:b}}{\includegraphics[width=1.6in]{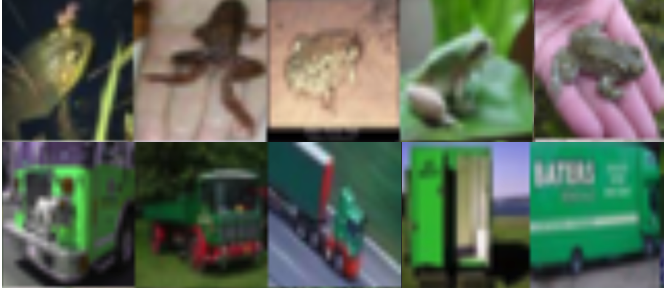}}
       \caption{Most normal (left) and most anomalous (right) in-class
       examples determined by OCSVM-SVDD for in which OCSVM-SVDD performs best.}
        \label{fig:ocsvmresults}
 \end{figure}

\subsection{ Detecting Adversarial attacks on GTSRB stop signs using OC-NN }
\textbf{Setup:}
In many applications (eg. autonomous driving) it is of paramount interest to effectively detect the adversarial samples to ensure safety and security. In this experiment, we
examine  the performance of proposed algorithms on  detecting
adversarial examples. We consider the “stop sign”
class of the German Traffic Sign Recognition Benchmark
(GTSRB) dataset, for which we generate adversarial examples
from randomly drawn stop sign images of the test set
using Boundary Attack~\cite{brendel2017decision}. We create training dataset consists of n = 1150 examples obtained by combining the normal and anomalous samples in both train and test samples. The  number of normal instances n = 1050 stop signs ( 780  from train set +  270 test set ) alongwith  100
adversarial examples are added to obtain training set. We pre-process the
data by removing the 10\% border around each sign, and then resize every image to $32 \times 32$ pixels following the same setup as in ~\cite{pmlrv80ruff18a}. Furthermore, we  apply global contrast normalization using the $L1-norm$ and rescale to the unit interval $[0, 1]$.\\
\textbf{Network architecture} We use a CNN with LeNet architecture having three convolutional modules, $16\times(5\times5\times3)$-filters, $32 \times (5 \times 5 \times 3)$-filters, and $64 \times (5 \times 5 \times 3)$-filters, followed by a final dense layer of 32 units. We train with a smaller batch size of 64, due to the dataset size and set again hyperparamter $\lambda = 10^{ - 6}$.\\
\textbf{Results:}
Results Table ~\ref{tab:gtsrbresults} illustrates the AUC scores obtained. The RCAE outperforms all the other deep models. Figure ~\ref{fig:gtsrbrcaeresults} and Figure ~\ref{fig:gtsrbocnnresults}  shows the most anomalous samples detected by RCAE and OCNN methods respectively, the outliers in this experiment are the images in odd perspectives and the ones that are cropped incorrectly.

\begin{table*}[!t]
   \caption{Average AUCs in \% with StdDevs (over 10 seeds) per method on MNIST and CIFAR-10 dataset.}
   \label{tab:ocnn_results}
   \small 
   \centering 
   \scalebox{0.6}{
   \begin{tabular}{lccccccccr} 
   \toprule[\heavyrulewidth]\toprule[\heavyrulewidth]
   \textsc{\pbox{20cm}{Normal \\ Class}} & \textsc{\pbox{20cm}{OCSVM / \\ SVDD}} & \textsc{KDE} &  \textsc{IF} & \textsc{DCAE} & \textsc{ANOGAN} & \textsc{\pbox{20cm}{SOFT-BOUND \\ DEEP SVDD}} & \textsc{\pbox{20cm}{ONE-CLASS \\ DEEP SVDD}} & \textsc{OC-NN} & \textsc{RCAE} \\
   \midrule
   0 & $96.75\pm0.5$ & $97.1\pm0.0$ & $95.32\pm1.2$  & $99.90\pm0.0$ & $96.6\pm1.3$ & $98.66\pm1.3$     & $97.78\pm0.0$ & $97.60\pm1.7$ &
   $\bf{99.92\pm0.0}$\\
   1 & $99.15\pm0.4$ & $98.9\pm0.0$ & $99.35\pm0.0$  & $99.96\pm2.1$ & $99.2\pm0.6$ & $99.15\pm0.0$     & $99.08\pm0.0$ & $\color[rgb]{0,0,1}99.53\pm0.0$ & $\bf{99.97\pm2.2}$\\
   2 & $79.34\pm2.2$ & $79.0\pm0.0$ & $73.15\pm4.5$  & $96.64\pm1.3$ & $85.0\pm2.9$ & $88.09\pm2.2$     & $88.74\pm1.2$ & $87.32\pm2.1$ & $\bf{98.01\pm1.2}$\\
   3 & $85.88\pm1.3$ & $86.2\pm0.0$ & $81.34\pm2.8$ & $98.42\pm0.0$ &  $88.7\pm2.1$ &  $88.93\pm3.4$     & $88.26\pm3.2$ & $86.52\pm3.9$ & $\bf{99.25\pm0.0}$\\
   4 & $94.18\pm1.5$ & $87.9\pm0.0$ & $87.40\pm2.6$  & $98.72\pm0.0$ & $89.4\pm1.3 $& $93.88\pm2.3$     & $95.24\pm1.4$ & $93.25\pm2.4$ & $\bf{99.23\pm0.0}$\\
   5 & $72.77\pm3.7$ & $73.8\pm0.0$ & $73.96\pm2.9 $ & $97.80\pm1.3$ & $88.3\pm2.9$ & $84.35\pm3.1$     & $83.76\pm3.1$ & $\color[rgb]{0,0,1}86.48\pm3.3$ &$ \bf{99.21\pm0.0}$\\
   6 & $95.14\pm1.1$ & $87.6\pm0.0$ & $88.65\pm0.0$  & $99.74\pm0.0$ & $94.7\pm2.7$ & $97.74\pm0.0$     & $97.99\pm0.0$ & $97.12\pm1.4$ & $\bf{99.81\pm1.1}$\\
   7 & $91.86\pm1.6$ & $91.4\pm0.0$ & $91.44\pm1.8$  & $99.08\pm2.2$ & $93.5\pm1.8$ & $92.60\pm1.2$     & $93.55\pm2.3$ & $93.64\pm2.1$ & $\bf{99.18\pm0.0}$\\
   8 & $88.65\pm1.2$ & $79.2\pm0.0$ & $75.10\pm3.7$  & $96.98\pm0.0$ & $84.9\pm2.1$ & $90.69\pm3.3$     & $90.25\pm3.1$ & $88.54\pm4.7$ & $\bf{98.50\pm2.2}$\\
   9 & $92.53\pm1.9$ & $88.2\pm0.0$ & $87.59\pm1.5$  & $98.04\pm1.3$ & $92.4\pm1.1$ & $94.28\pm2.5$     & $94.12\pm2.4$ & $93.54\pm3.3$ & $\bf{98.98\pm1.3}$\\
   \hline
  \textsc{Aeroplane}    & $60.37\pm1.3$ & $61.2\pm0.0$ & $63.99\pm1.1$  &$71.21\pm1.5$& $67.1\pm2.5$ & $66.15\pm1.1$ & $67.33\pm2.2$ & $60.42\pm1.9$  & $\bf{72.04\pm2.5}$\\
   \textsc{Automobile}  & $63.03\pm1.4$ & $63.0\pm0.0$ & $60.56\pm1.0$  &$63.05\pm2.3$& $54.7\pm3.4 $& $57.64\pm3.2$ & $58.14\pm3.1$ & $\color[rgb]{0,0,1}61.97\pm2.0$  & $\bf{63.08\pm2.1}$\\
   \textsc{Bird}        & $63.47\pm1.0 $& $50.1\pm0.0$ & $64.51\pm1.1$  &$71.50\pm1.1$& $52.9\pm3.0 $& $61.99\pm1.4$ & $61.35\pm1.5$ & $\color[rgb]{0,0,1} {63.66\pm1.4}$ & $\bf{71.67\pm1.3}$\\
   \textsc{Cat}         & $60.25\pm0.9$ & $56.4\pm0.0$ & $56.16\pm2.6$  &$60.57\pm0.0$& $54.5\pm1.9$ & $57.56\pm4.1$ & $55.72\pm1.4$ & $53.57\pm2.1$  & $\bf{60.63\pm1.1}$\\
  \textsc{Deer}         & $69.15\pm0.8$ & $66.2\pm0.0 $& $72.66\pm1.1 $ &$70.85\pm2.2$& $65.1\pm3.2$ & $63.36\pm1.3$ & $63.32\pm1.2$ & $\color[rgb]{0,0,1}67.40\pm1.7 $ & $\bf{72.75\pm3.3}$\\
   \textsc{Dog}         & $66.24\pm1.5$ & $62.4\pm0.0$ & $61.46\pm2.7 $ &$62.74\pm2.1$& $60.3\pm2.6$ & $58.58\pm1.2$ & $58.68\pm1.4$ & $56.11\pm2.1$  & $\bf{63.96\pm3.3}$\\
   \textsc{Frog}        & $\bf{71.57\pm1.5}$ & ${71.3\pm0.0}$ & $68.06\pm2.1 $ &$65.17\pm4.1$& $58.5\pm1.4 $& $63.93\pm3.1 $& $64.45\pm2.1$ & $63.31\pm3.0$  & ${64.88\pm4.2}$\\
   \textsc{Horse}       & $63.38\pm0.8 $& $62.6\pm0.0$ & $63.04\pm1.5$  &$61.11\pm1.4$& $62.5\pm0.8$ & $60.20\pm2.2$ & $59.80\pm2.6$ & $60.09\pm2.7$  &
   $\bf{63.64\pm0.0}$\\
  \textsc{Ship}         & $60.44\pm1.1$ & ${65.1\pm0.0}$ & $68.01\pm1.3$  &$74.18\pm1.2$& $74.68\pm4.1$ & $70.21\pm1.1$ & $67.44\pm2.2$ & $64.67\pm1.6$  & $\bf{74.72\pm1.1}$\\
  \textsc{Truck}        & $\bf{75.81\pm0.8}$ & ${74.0\pm0.0}$ & $72.83\pm1.1$  &$71.36\pm4.3$& $66.5\pm2.8 $& $72.91\pm3.3$ & $68.03\pm3.2$ &$ 60.32\pm4.9 $ & ${74.47\pm1.5}$\\
   \bottomrule[\heavyrulewidth]
   \end{tabular}}
\end{table*}

\begin{table*}[!t]
\caption{Average AUCs in \% with StdDevs (over 10 seeds) per
method on GTSRB stop signs with adversarial attacks.} 
\centering 
\scalebox{0.7}{
\begin{tabular}{l l } 
\hline\hline 
\bf{\textsc{Method}} & \bf{\textsc{AUC}}\\ [0.5ex] 
\hline 
OC-SVM/SVDD & $52.5 \pm 1.3$ \\ 
KDE & $51.5 \pm 1.6$ \\
IF & $53.37\pm 2.9$ \\
AnoGAN & - \\
DCAE & $79.1\pm3.0$ \\
SOFT-BOUND. DEEP SVDD & $67.53\pm4.7$ \\
ONE-CLASS. DEEP SVDD & $67.08\pm4.7$ \\
OC-NN & $63.53\pm2.5$ \\
\bf{\bf{RCAE $\lambda =0$}} &\bf{$87.45\pm3.1$} \\
\bf{RCAE} & {$\bf{87.39 \pm 2.7}$} \\ [1ex] 
\hline 
\end{tabular}}
\label{tab:gtsrbresults}
\end{table*}

\begin{figure}[htp]
      \centering
      \subcaptionbox{Top 50 Normal.\label{fig3:a}}{\includegraphics[width=1.6in,height=1.6in]{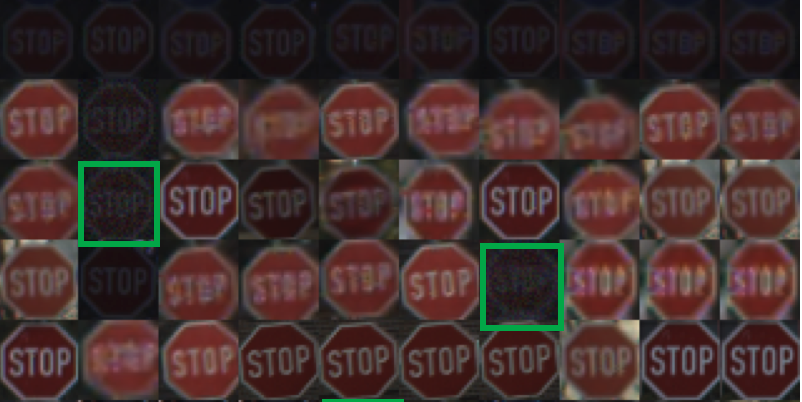}}\hspace{1em}%
      \subcaptionbox{Top 50 anomalies.\label{fig3:b}}{\includegraphics[width=1.6in,height=1.6in]{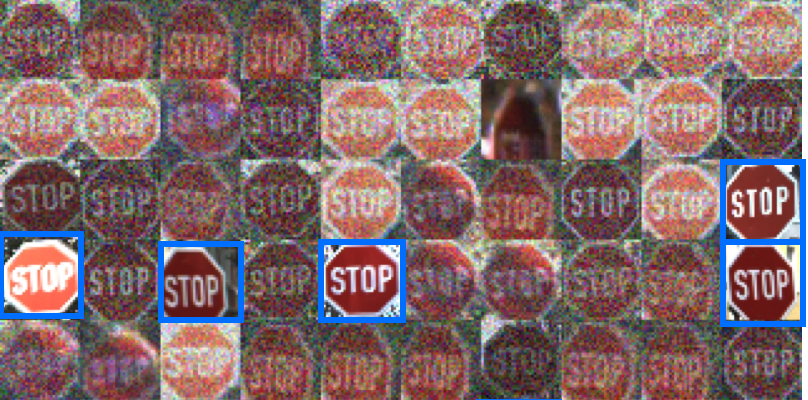}}
      \caption{Most normal and anomalous stop signs detected by RCAE. Adversarial examples are highlighted in green, Normal samples are highlighted in blue}
      \label{fig:gtsrbrcaeresults}
   \end{figure}

\begin{figure}[htp]
      \centering
      \subcaptionbox{Top 50 Normal.\label{fig3:a}}{\includegraphics[width=1.6in,height=1.6in]{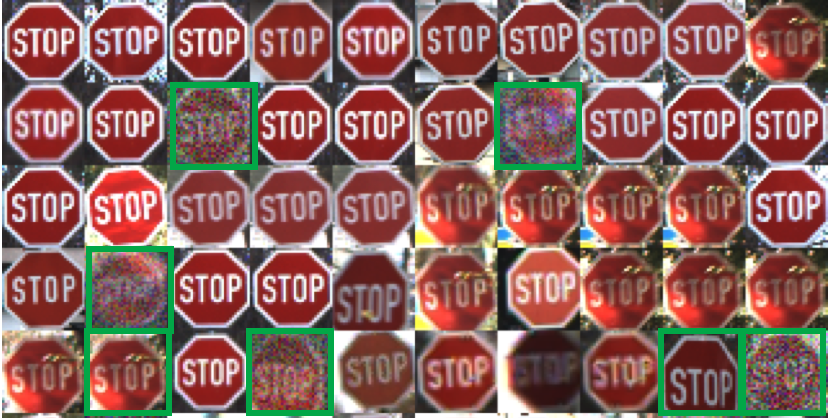}}\hspace{1em}%
      \subcaptionbox{Top 50 anomalies.\label{fig3:b}}{\includegraphics[width=1.6in,height=1.6in]{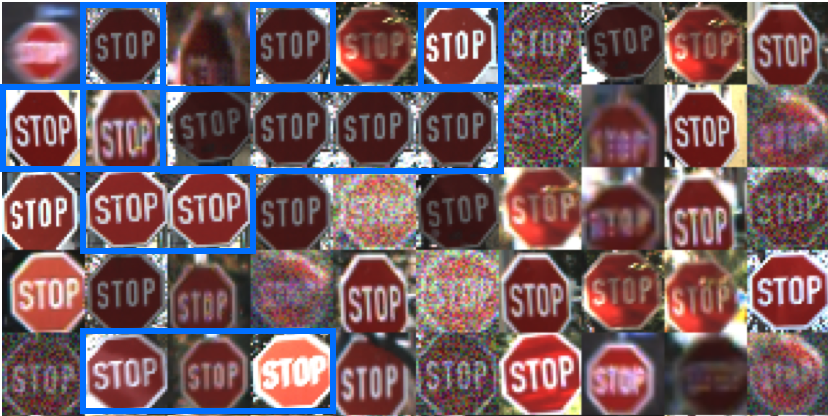}}
      \caption{Most normal and anomalous stop signs detected by OC-NN.  Adversarial examples are highlighted in green, Normal samples are highlighted in blue}
      \label{fig:gtsrbocnnresults}
      \end{figure}

\section{Conclusion}
\label{sec:conclusion}
In this paper, we have proposed a one-class neural network (OC-NN) approach for anomaly detection.
OC-NN uses a one-class SVM (OC-SVM) like loss function to train a neural network.
The advantage of OC-NN is that the features of the hidden layers are constructed for the specific
task of anomaly detection. This approach is substantially different from recently proposed
hybrid approaches which use deep learning features as input into an anomaly detector.  Feature
extraction in hybrid approaches is generic and not aware of the anomaly detection task. To learn
the parameters of the OC-NN network we have proposed a novel alternating minimization approach and
have shown that the optimization of a subproblem in OC-NN is equivalent to a quantile selection

\bibliographystyle{unsrtnat}
\bibliography{main}

\end{document}